%% file: main.tex
\documentclass[letterpaper, 10 pt, conference]{ieeeconf}

\IEEEoverridecommandlockouts
\usepackage{graphics} % for pdf, bitmapped graphics files
\usepackage{amsmath} % assumes amsmath package installed
\usepackage{amssymb}  % assumes amsmath package installed
\usepackage[table]{xcolor}
\usepackage[ruled,linesnumbered, noend]{algorithm2e}

\SetCommentSty{mycommfont}
\usepackage{soul}
\usepackage{cite}
\usepackage{comment}
\usepackage{multirow}
\usepackage{graphicx}
\usepackage{wrapfig}
\usepackage{subcaption}
\usepackage{balance}
\usepackage{soul, color}
\usepackage{wrapfig}
\usepackage{amsthm}
\newtheorem{lem}{Lemma}
\newtheorem{mydef}{Definition}
\newtheorem{thm}{Theorem}

\newtheorem{pro}{Proposition}

\renewcommand{\baselinestretch}{0.965}
\title{\LARGE \bf Constant-time Motion Planning with Anytime Refinement for Manipulation}

\author{Itamar Mishani\authorrefmark{1}, Hayden Feddock\authorrefmark{2}, Maxim Likhachev\authorrefmark{1}
\thanks{I. Mishani and M. Likhachev are with the Robotics Institute, School of Computer Science, Carnegie Mellon University, Pittsburgh, PA. e-mail: {\tt \{imishani, mlikhach\}@.cs.cmu.edu}.}
\thanks{H. Feddock is with Electrical and Computer Engineering Department, Pittsburgh University. e-mail: {\tt \{htf1\}@.pitt.edu}.}
\thanks{Research was sponsored by the ARM (Advanced Robotics for Manufacturing) Institute through a grant from the Office of the Secretary of Defense and was accomplished under Agreement Number W911NF-17-3-0004. The views and conclusions contained in this document are those of the authors and should not be interpreted as representing the official policies, either expressed or implied, of the Office of the Secretary of Defense or the U.S. Government. The U.S. Government is authorized to reproduce and distribute reprints for Government purposes notwithstanding any copyright notation herein.}
}

\begin{document}
\setlength{\textfloatsep}{1.5pt}
% \markboth{IEEE Robotics and Automation Letters. Preprint Version. Accepted November, 2021}
% {Mishani \MakeLowercase{\textit{et al.}}: Real-time Non-visual Shape Estimation and Manipulation Control of an Elastic Wire}

\maketitle

\input{abstract}

% \begin{IEEEkeywords}
%  Elastic wires, Dual Arm Manipulation, Manipulation Planning.
% \end{IEEEkeywords}

\input{introduction}

\input{related_work}

\input{problem_formulation}

\input{algorithmic_approach}

\input{assembly_implementation}
\input{experiments}

\input{conclusions}

% \balance
\bibliographystyle{IEEEtran}
\bibliography{main}

\end{document}

%% file: abstract.tex
\begin{abstract}

    Robotic manipulators are essential for future autonomous systems, yet limited trust in their autonomy has confined them to rigid, task-specific systems. The intricate configuration space of manipulators, coupled with the challenges of obstacle avoidance and constraint satisfaction, often makes motion planning the bottleneck for achieving reliable and adaptable autonomy. Recently, a class of constant-time motion planners (CTMP) was introduced. These planners employ a preprocessing phase to compute data structures that enable online planning provably guarantee the ability to generate motion plans, potentially sub-optimal, within a user defined time bound. This framework has been demonstrated to be effective in a number of time-critical tasks. However, robotic systems often have more time allotted for planning than the online portion of CTMP requires, time that can be used to improve the solution. To this end, we propose an anytime refinement approach that works in combination with CTMP algorithms. Our proposed framework, as it operates as a constant time algorithm, rapidly generates an initial solution within a user-defined time threshold. Furthermore, functioning as an anytime algorithm, it iteratively refines the solution's quality within the allocated time budget. This enables our approach to strike a balance between guaranteed fast plan generation and the pursuit of optimization over time. We support our approach by elucidating its analytical properties, showing the convergence of the anytime component towards optimal solutions. Additionally, we provide empirical validation through simulation and real-world demonstrations on a 6 degree-of-freedom robot manipulator, applied to an assembly domain.

\end{abstract}

%% file: introduction.tex
\section{Introduction}

Manipulation is prevalent across a wide range of applications. It spans from residential to warehouse and factory environments, where tasks often involve a degree of environmental stability and repetition. Recently, a new concept of Constant-Time Motion Planning (CTMP) was introduced, with the aim of guaranteeing the ability to generate a motion plan within a (short) constant time in such environments \cite{ctmp, ConveyerCTMP, APP}. CTMP has proven to be highly valuable in time-critical applications, such as manipulators catching incoming projectiles, operating on conveyor belts, and mail sorting in a mailroom. Although these algorithms frequently yield solutions within mere milliseconds, in numerous domains the allotted time bound for planning extends beyond these limits, thus remaining underutilized.

Our inspiration is drawn from manufacturing settings that rely on robotic manipulators for assembly tasks. Common approaches in such settings tend to confine the system to a task-specific configurations, employing a predefined set of motion plans, often designed by humans, thereby constraining autonomy. For example, consider the assembly cell illustrated in Fig. \ref{fig:cell}, where three robotic arms continuously pick and place objects in a \textit{semi-static} environment\cite{APP}. In such an environment, while the general settings remain constant, the manipulated objects can move within their designated regions. This context demands real-time, reliable, and efficient motion planning to instill confidence in autonomous operations. Moreover, as the assembly task involves a sequence of interdependent operations, each building upon previous one, the need for rapid and high quality plan queries becomes crucial.

% Our inspiration is drawn from manufacturing settings, where robotic manipulators are integral to assembly tasks. However, these tasks commonly necessitate specialized part feeders, exact fixtures, and intricate robot programming, incorporating motion plans designed by humans. This complex nature poses challenges for adaptability and reusability, particularly in scenarios involving high product diversity and low production volumes \cite{HMLV}.

% Consider an assembly cell depicted in Fig. \ref{fig:cell} as a running example, where three robotic arms engage in a continuous process of picking and placing objects within a \textit{semi-static} environment\cite{APP}. In such an environment, the general settings remain static, while the manipulated objects are permitted to move within their designated manipulation regions. This context demands real-time, reliable, and efficient motion planning to instill confidence in autonomous operations. Furthermore, given that the assembly task is composed of a sequence of interdependent operations, each building upon previous actions, a pressing need arises for rapid and nearly optimal plan queries.

% , ultimately leading to the complete assembly of a product. While the overall environment remains static, the manipulated objects have the potential to shift within and between predefined regions, thus defining 

%%%%%%%%%%%%%%%%%%%%%%%%%%%%%%
\begin{figure}
\centering
\includegraphics[width=0.49\linewidth]{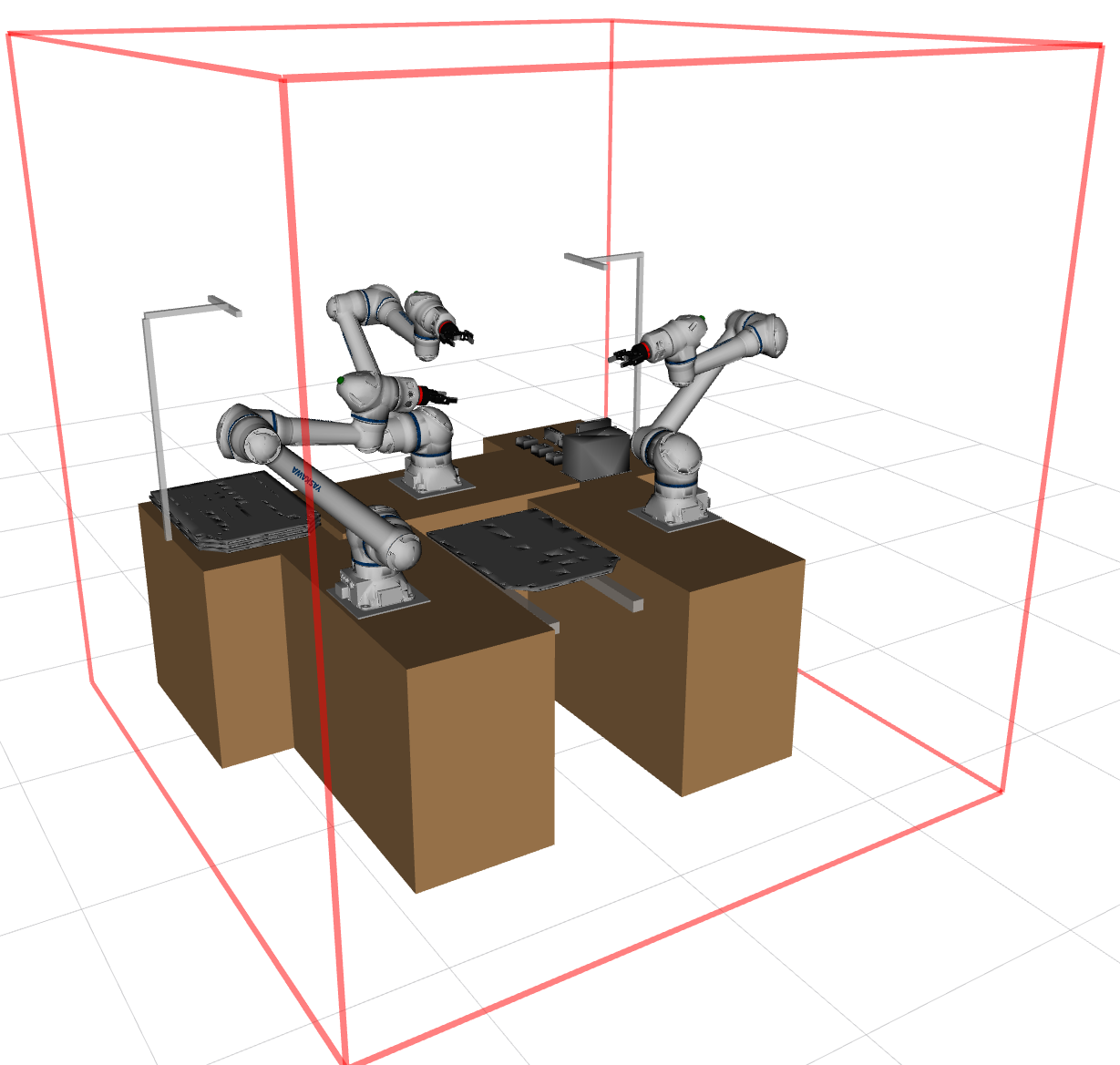}
\hfill
\includegraphics[width=0.49\linewidth]{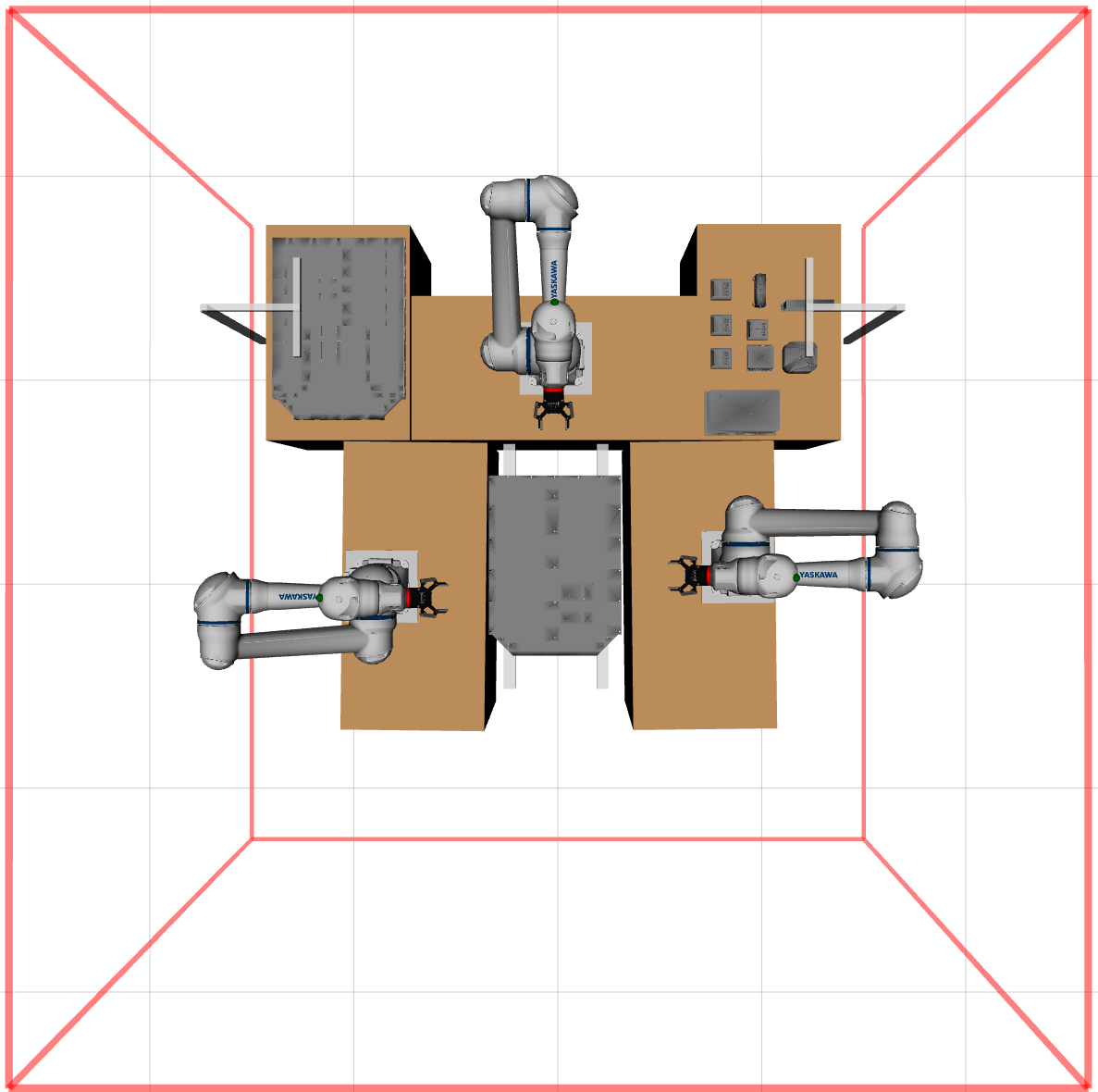}
\caption{An assembly cell, composed of three Yaskawa HC10DTP manipulators.}
\label{fig:cell}
\end{figure}
%%%%%%%%%%%%%%%%%%%%%%%%%%%%%%

CTMP approaches employ offline preprocessing to generate auxiliary data structures, which are subsequently utilized in an online query phase to compute paths in constant time. However, they often yield suboptimal solutions. While guaranteeing path generation within milliseconds, the available time budget is frequently longer, offering opportunities for solution improvements. Additionally, in cases of continuous manipulation with successive queries that do not require an arm to return to its home configuration (such as in our assembly example), CTMP typically generates highly sub-optimal plans due to its reliance on a pre-defined set of possible start configurations. 

Motivated by these challenges, we introduce an anytime variant of these algorithms. This variant employs a preprocessing phase, similar to \cite{ctmp} and \cite{ConveyerCTMP}, that enables the rapid generation of an initial solution — though potentially quite suboptimal — within a constant time constraint. Unlike previous approaches though, the algorithm progressively improves solution quality while adhering to the defined time budget, which can be determined by user-defined limits or based on previous execution times.
In the subsequent sections, we describe the algorithm, provide theoretical analysis of the proposed anytime refinement showing that it converges to an optimal solution if time permits, and demonstrate its utility in the domain of assembly both in simulation and on a physical robot.

%% file: related_work.tex
\section{Related Work}
\subsection{Preprocessing-based planning}
\label{rel:A}
Preprocessing-based planning techniques are frequently used to offer real-time planning advantages. A common algorithm for doing so is the Probabilistic Roadmaps (PRM), which efficiently preprocesses known environments by generating a dense roadmap, allowing for quick online query responses \cite{PRM}. Nevertheless, queries might not always connect to the roadmap due to PRM's asymptotic guarantees \cite{PRM_guarantees}. Additionally, the connection process relies on the main bottle-neck in manipulation planning: collision-detection, which is computationally a very expensive process. Recent work has focused on decomposing the configuration space into a set of collision-free convex sets \cite{polyhedral-convex, np-convexsets}, and subsequently finding smooth trajectories within these sets using optimization methods \cite{marcucci2023shortest, mp_aroundCS, geodesic_gcs}. However, since these methods solve an optimization problem over the entire space for each query, they do not provide constant planning time bounds.

An alternative class of preprocessing-based methods rely on past experiences to speed up search \cite{experience_ken, experience_nik, egraphs}. While speeding up search, none of the aforementioned algorithms can provably ensure fixed planning-time guarantees.

Recently, a planning approach with provable constant-time guarantees (CTMP) has been introduced \cite{ctmp, APP, ConveyerCTMP}. In all its variations, this method relies on a preprocessing phase where a library of intelligently computed paths is utilized during the online phase. In particular, \cite{ctmp} introduces a framework for decomposing a pre-defined region-of-interest (RoI) within a static environment into a set of sub-regions. By computing only one representative path per each sub-region, it enables an online planner to provably guarantee finding a plan to any state within the entire RoI within a user-defined (small) time bound. This concept was subsequently expanded to encompass static environments featuring dynamic goal objects \cite{ConveyerCTMP}, as well as semi-static environments \cite{APP}. However, these methods usually generate solutions within a few milliseconds and do not exploit the entire time budget available to them during an online phase. Furthermore, although beneficial for generating fast solutions, these planners operate under the assumption of a predefined set of home configurations (start states) and disregard task sequences. This limitation forces successive plans to traverse the same state, frequently resulting in highly suboptimal paths for manipulation tasks like assembly operations.

\subsection{Anytime Heuristic Search}
\label{rel:B}

In scenarios where planning time is restricted, an anytime search algorithm, becomes advantageous, as it aims to rapidly compute an initial, possibly suboptimal solution, and subsequently refine it as time permits. To achieve anytime characteristics in heuristic search, typical approaches involve obtaining fast initial solution through algorithms like beam search, and then progressively increase the number of expandable nodes at a given depth level \cite{anytime_beam, incremental_beam_search, triangle_beam, awa*}. Alternatively, and widely used in robotics, anytime algorithms often incrementally adjust the heuristic inflation \cite{anytime_dynamicA*, anytime_heu, anytime_focal, ana*, anytime_egraphs, amha*}.
Nonetheless, none of the above provide planning time bounds for the initial solution. 

\subsection{Assembly Planning}
\label{rel:C}

The assembly planning problem involves devising a systematic and efficient sequence of actions and motions required to assemble a complex product from individual components. 
Many such tasks fall under Task and Motion Planning (TAMP), which integrates high-level task planning and low-level motion planning \cite{Hartmann_rearrange, tamp}. To address assembly's sequential nature, research often targets optimal sequences \cite{Hartmann_rearrange, wan2016integrated, hartmann2020robust}. However, typically used motion planners, such as Rapidly-Exploring Random Trees (RRTs) \cite{rrt, RRT*, RRT-Connect} and PRMs \cite{PRM}, overlook this repetitive and sequential nature. Here, we focus on developing a motion planner that provides completeness and constant-time planning guarantees, using the repetitive nature of the assembly problem for generating fast feasible solutions that are progressively being optimized until the permitted planning time expires..

%% file: problem_formulation.tex
\section{Preliminary}
\label{Prob}
% , excluding dynamics or underactuated systems
Consider $\mathcal{X}$ as the state space of a robot $\mathcal{R}$ operating within a semi-static environment $\mathcal{W} \subset \mathbb{R}^3$.
We assume $\mathcal{R}$ to be controllable and focus on kinematic planning. Consequently, plans are presumed reversible and feasible in both directions.
Let $\mathcal{G} \subseteq \mathcal{X}$ denote a region-of-interest (RoI), with $\mathcal{G} = \bigcup\limits_{i=1}^n \mathcal{G}_i$ which may potentially consist of a set of disjoint goal regions that we call \textit{local-RoIs} $\mathcal{G}_i$. In an assembly context, these local-RoIs might represent picking or placing regions. Given a goal state $g \in \mathcal{G}$ (e.g., grasping, placing), our objective is to plan the robot's motion to it. We aim to generate these motions quickly within an upper-bound time constraint $T_{bound}$. However, a trade-off often exists between planning time and path optimality. We note two key observations: Firstly, instances arise where a plan can be computed within a timeframe faster than $T_{bound}$, affording room to refine the solution using the time difference. Secondly, in domain with sequential nature, the time constraint may prove flexible. If an ongoing path remains incomplete, and there is $T_{left}$ until its execution concludes, the planning time limit can be extended by $T_{left}$.

We formulate the planning problem as a graph search by discretizing the planning space and representing the RoI (goal region) as a set of states within the graph located in that region. 

%%%%%%%%%%%%%%%%%%%%%%%%%%%%%%
\begin{figure}[!ht]
\centering
\includegraphics[width=0.48\textwidth]{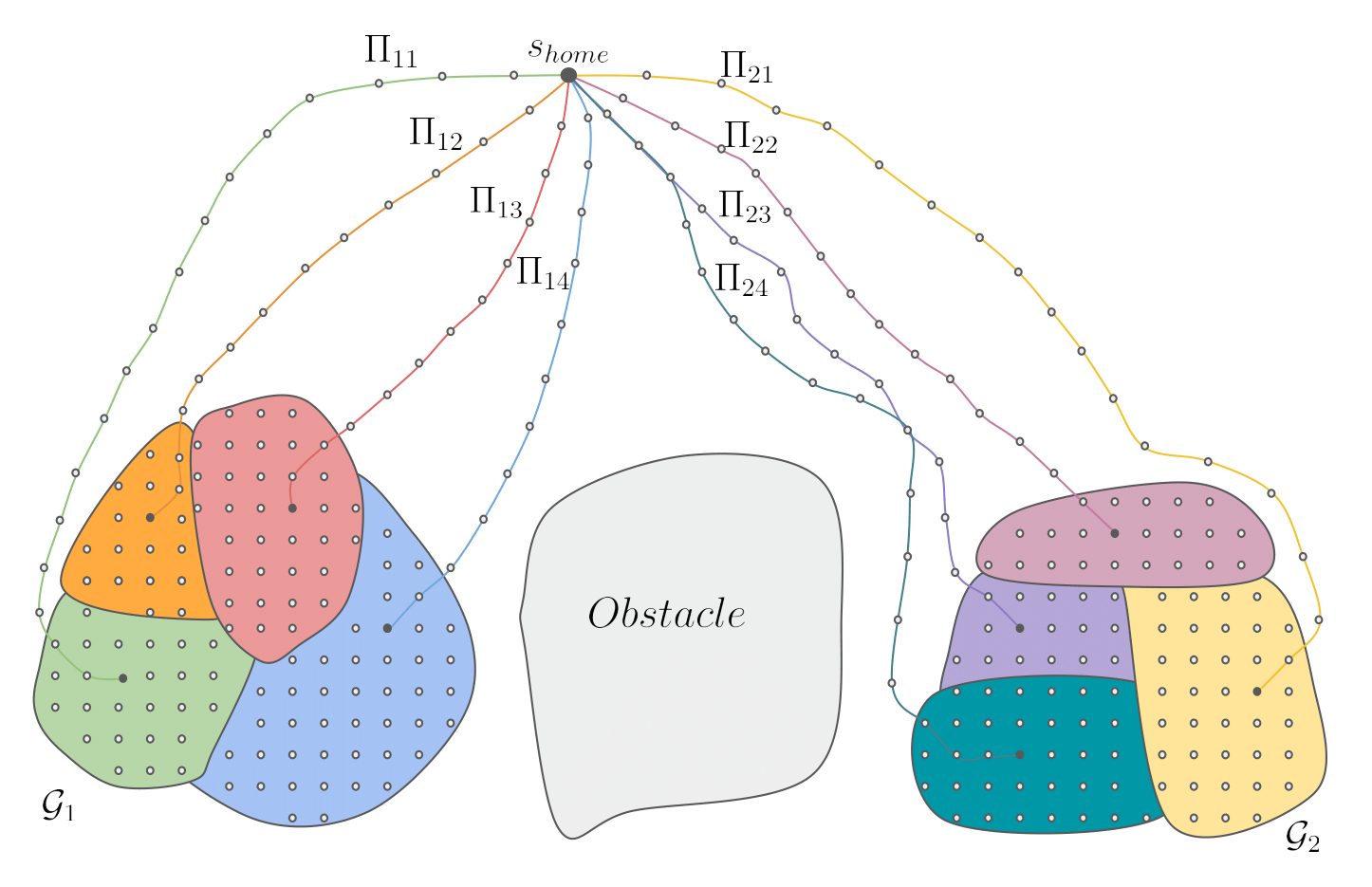}
\caption{Illustration of the preprocessing phase. For each local-RoI $\mathcal{G}_i$, we compute neighborhoods, distinguished by different colors, to finally form the cover of $\mathcal{G}$. Moreover, each neighborhood is associated with a path $\Pi_{ij}$ from $s_{home}$ to an \textit{attractor state}, ensuring that every state within that neighborhood is constant-time feasible from the attractor state.}
\label{fig:preprocess}
\end{figure}
%%%%%%%%%%%%%%%%%%%%%%%%%%%%%
Previously presented in \cite{ctmp, ConveyerCTMP}, we re-introduce the notation of \textit{reachable} states.
\begin{mydef}
\label{def:rechable}
    A goal state $g \in \mathcal{G}$ is \textbf{reachable} from a state $s$ if there exists a path from $s$ to $g$ and it can be computed in finite time.
\end{mydef}
We wish for our framework to be able to plan a path from any possible state that the robot can be at, to any reachable goal state $g \in \mathcal{G}$ within $T_{bound}$. Thus, we first define the following:

\begin{mydef}
\label{def:ct_feasible}
    We say that a reachable state $g \in \mathcal{G}$ is \textbf{constant-time feasible} from a state $s$ if a planner can find a path to it within $T_{bound}$.
\end{mydef}

To achieve this, in an offline phase we decompose our RoI $\mathcal{G}$ into a collection of sets referred to as the \textit{cover}, comprising sub-regions called \textit{neighborhoods}. Formally, we define them as follows:

\begin{mydef}
    A \textbf{cover} of $\mathcal{G}$ denoted as $\mathcal{O}$, is a collection of sets $o_i$ each called a \textbf{neighborhood} for which:
    \begin{itemize}
        \item $o_i \subset \mathcal{X}$ for each $o_i \in \mathcal{O}$.
        \item $\mathcal{G} \subseteq \bigcup\limits_{o_i \in \mathcal{O}} o_i$.
    \end{itemize}
\end{mydef}

Our aim is to use this decomposition to ensure constant-time feasibility for all reachable states within each neighborhood in online settings.

%% file: algorithmic_approach.tex
\section{Algorithmic Approach}
\label{alg_app}

% The algorithmic approach is organized into two sections. The first section introduces the preprocessing phase, which constitutes the core of the CTMP framework. The second section outlines the online phase, wherein the anytime refinement approach is employed.

\subsection{Preprocessing Phase}
\label{alg_app:preprocess}

Our approach to obtaining an initial solution within a constant-time framework is built upon a preprocessing phase that enables efficient online computations. Similarly to \cite{ctmp, ConveyerCTMP, APP}, during the offline phase, we assume access to a planner $\mathcal{P}$ and a predefined home state $s_{home}$ (the framework can also be extended to support a finite set of home configurations). A straightforward approach might involve computing and storing paths from $s_{home}$ to all states in $\mathcal{G}$ and all feasible paths between any $s_i, s_j \in \mathcal{G}$ pairs. However, this approach demands excessive memory usage. Alternatively, as depicted in Fig. \ref{fig:preprocess}, we can construct neighborhoods around specific states, called \textit{attractor states}, to form a cover $\mathcal{O}$ of $\mathcal{G}$. We construct these neighborhoods such that each state within them can easily reach an attractor state through a path computable within a (small) bounded time.
One method, similar to \cite{ctmp}, involves creating the neighborhood with a navigation function (e.g., Euclidean distance) such that from any state within the neighborhood, always transitioning to a state that maximally reduces the specified navigation function guarantees a path leading to the attractor state. Another method, akin to \cite{ConveyerCTMP}, guarantees that the path corresponding to a neighborhood's attractor state can be utilized as an experience, enabling constant-time planning to all states within the neighborhood. In this work, we adopt the first method and independently preprocess all $\mathcal{G}_i$, with a shared home state $s_{home}$ as described in Alg. \ref{alg:preprocess}.

\begin{algorithm}[!ht]
    \caption{Preprocess}
    \label{alg:preprocess}
    \footnotesize
    \SetKwInOut{Input}{Input}
    \SetKwInOut{Output}{Output}
    \SetKwFunction{ObjectPathPairs}{\scriptsize ObjectPathPairs}
    \SetKwFunction{SampleValidUncovState}{\scriptsize SampleValidUncovState}
    \SetKwFunction{ConstrcutNeighborhood}{\scriptsize ConstrcutNeighborhood}
    \SetKwFunction{InitHashMap}{\scriptsize InitHashMap}
    \Input{$s_{home}$: Home state for all local-RoIs \newline
        % $\mathcal{R}$: Robot model \newline
        $\mathcal{P}$: Planner \newline
        $\mathcal{G}$: RoI }% \newline
        % $objMaxPrims$: A set of object primitives at their maximum allowed size \newline
        % $\delta$: an increment parameter for the any-object process} 
    \Output{Library $\mathcal{L}$ maps local-RoIs to their cover, containing neighborhoods and their corresponding paths}
    \SetAlgoLined\DontPrintSemicolon
    \SetKwFunction{proc}{Preprocess}
    \SetKwProg{myproc}{Procedure}{}{}
    \myproc{\proc{$s_{home}$, $\mathcal{P}$, $\mathcal{G}$}}{
    $\mathcal{L} \leftarrow $ \InitHashMap() \tcp*{A hashmap from a local-RoI to its data} 
    \For{$\mathcal{G}_i$ in $\mathcal{G}$}{
        $\mathcal{G}_i^{covered} \leftarrow \emptyset$ \;
        \While{$\mathcal{G}_i \setminus \mathcal{G}_i^{covered} \neq \emptyset$}{
            $s_{attractor} \leftarrow $ \SampleValidUncovState($\mathcal{G}_i$)\;
            $\Tilde{\Pi} = \mathcal{P}.PlanPaths(s_{attractor}, s)$  \tcp*{Compute a set of representative paths}
            \If{$\Tilde{\Pi} = \emptyset$ \tcp*{No valid path}}{continue\;}
            $o \leftarrow$ \ConstrcutNeighborhood($s_{attractor}$) \; 
            $\mathcal{G}_i^{covered} \leftarrow \mathcal{G}_i^{covered} \bigcup o$ \;
            
            % \tcc{Compute guaranteed collision-free paths with object primitives}
            % $\Tilde{\Pi} \leftarrow $ \ObjectPathPairs($\mathcal{R}$, $s_{home}$, $\mathcal{P}$, $\Pi_0$, $objMaxPairs$, $\delta$) \tcp*{Object-path pairs container}
            $\mathcal{L}[\mathcal{G}_i] \leftarrow \mathcal{L}[\mathcal{G}_i] \bigcup (s_{attractor}, o, \Tilde{\Pi})$\;
        }
    }
    \KwRet $\mathcal{L}$\;}
\end{algorithm}

The algorithm iterates through all local-RoIs ($\mathcal{G}_i$), computing the neighborhoods to add to the cover (line 3). For every local-RoI, it samples a valid yet uncovered state (line 6). If a feasible path exists to this state, the algorithm chooses it as attractor state and expands a neighborhood around it (line 11). Adhering to the approach outlined in \cite{ctmp}, during the expansion of the neighborhood, we look for all states that can construct a path to the attractor state by always moving to a successor that decreases the navigation function the most. Additionally, the sampling process can be made informed and efficient, rather than being completely random. We continuously track the frontier of the expanded neighborhood during its construction. Once a neighborhood is returned and the sampler is invoked again, the cached frontier set comprises valid states that lie outside the bounds of the previously constructed neighborhood. Subsequently, we attempt to sample the next attractor state from this frontier set (More details can be found in \cite{ctmp}).

\subsection{Online Phase}
\label{alg_app:online}

\subsubsection{Initial Solution}
During the online phase, our aim is to leverage preprocessed data to rapidly generate initial plans within a specified time limit. The core concept is straightforward: utilizing a lookup approach, we can associate the queried point $g \in G$ with its corresponding neighborhood and ascertain the path from $s_{home}$ to $s_{attractor}$ associated with this specific neighborhood followed by connecting $g$ to $s_{attractor}$. 
The connection from g to $s_{attractor}$ corresponds to following a sequence of states from $g$ that minimize the navigation function w.r.t. $s_{attractor}$ (e.g., decreasing Euclidean distance w.r.t. $s_{attractor}$). Having the assumption of reversibility, we concatenate the path from $s_{home}$ to $s_{attractor}$ with the reversed path from $s_{attractor}$ to $g$.

Having access to preprocessed data where all valid states $g \in \mathcal{G}$ are constant-time feasible from $s_{home}$ we claim the following:

\begin{mydef}
    Consider the collection of all states that are constant-time feasible within $\mathcal{G}$ from $s_{home}$:
    $$\mathcal{F} = \{g \in \mathcal{G} | \textit{ g is constant-time feasible from } s_{home}\}$$ 
    Furthermore, let $\rho$ denote the set of all states across the precomputed paths:
    $$ \rho = \{s \in \bigcup\limits_{i=1}^n \bigcup\limits_{j} \Pi_{ij}\} $$  
    The set of all robot \textbf{potential states} is defined as:
    $$\Phi = \mathcal{F} \bigcup \rho$$
\end{mydef}

\begin{pro}
\label{prop:all_connected}
    All reachable goal poses $g \in \mathcal{F}$ are constant-time feasible from any potential state $s \in \Phi$.
\end{pro}
Since both states $s\in \Phi$ and $g \in G$ are constant-time feasible from $s_{home}$, and considering the reversibility of paths, we can directly concatenate one path with the reverse of the other path. 

\begin{algorithm}[!ht]
    \caption{Query}
    \label{alg:query}
    \footnotesize
    \SetKwInOut{Input}{Input}
    \SetKwInOut{Output}{Output}
    \SetKwFunction{FindRepPath}{\scriptsize FindRepPath}
    \SetKwFunction{Connect}{\scriptsize Connect}
    \SetKwFunction{GetCurrentTime}{\scriptsize GetCurrentTime}
    \SetKwFunction{Reverse}{\scriptsize Reverse}
    \SetKwFunction{ExtractPath}{\scriptsize ExtractPath}
    \SetKwFunction{AnytimeRefinement}{\scriptsize AnytimeRefinement}
    
    \Input{$s$: start state ($s \in \Phi$) \newline
        $g$: goal state ($g \in \mathcal{G}$) \newline
        $s_{home}$: Home state \newline
        $T_{bound}$: Time budget \newline
        $\mathcal{L}$: The preprocessed library}
    \Output{Path $\Pi$ from start state to goal state}
    \SetAlgoLined\DontPrintSemicolon
    \SetKwFunction{proc}{Query}
    \SetKwProg{myproc}{Procedure}{}{}
    \myproc{\proc{$s$, $g$, $s_{home}$, $T_{bound}$, $\mathcal{L}$}}{
    $T_{start} = \GetCurrentTime()$\;
    \eIf{$\Pi_{home, g} \leftarrow $ \FindRepPath($g$, $\mathcal{L}$) $\neq \emptyset$ \tcp*{Find the local-RoI, the specific neighborhood and the corresponding representative path}}{
        $\Pi_{home, g} \leftarrow \Connect(\Pi_{home, g}, g)$ \tcp*{Connect $g$ to the path (e.g., greedy search)}
        \eIf{$s \neq s_{home}$}{
            $\Pi_{home, s} \leftarrow$ \FindRepPath($s$, $\mathcal{L}$) \;
            \If{$\Pi_{home, s} = \emptyset$}{
                \KwRet $\emptyset$\;
            }
            $\Pi_{home, s} \leftarrow \Connect(\Pi_{home, s}, s)$\;
            $\Pi = \Reverse(\Pi_{home, s}) \cdot \Pi_{home, g}$ \tcp*{Concatenate the two paths}
        }{ $\Pi = \Pi_{home, g}$}
        $\Pi \leftarrow \AnytimeRefinement(s, g, \Pi, T_{start}, T_{bound})$\;
        \KwRet $\Pi$\;        
    }{
        \KwRet $\emptyset$\;
    }
}

\end{algorithm}

\begin{algorithm}
    \caption{Anytime Refinement}
    \label{alg:refine}
    \footnotesize
    \SetKwInOut{Input}{Input}
    \SetKwInOut{Output}{Output}
    \SetKwFunction{FindRepPath}{\scriptsize FindRepPath}
    \SetKwFunction{Connect}{\scriptsize Connect}
    \SetKwFunction{GetCurrentTime}{\scriptsize GetCurrentTime}
    \SetKwFunction{Reverse}{\scriptsize Reverse}
    \SetKwFunction{ExtractPath}{\scriptsize ExtractPath}

    \Input{$s_{start}$: start state ($s_{start} \in \Phi$) \newline
    $s_{goal}$: goal state ($s_{goal} \in \mathcal{G}$) \newline
    $\Pi$: Initial path \newline
    $T_{start}$: Start time of the query \newline
    $T_{bound}$: Time budget given}
    \Output{Refined Path $\Pi$}
    \SetAlgoLined\DontPrintSemicolon
    \SetKwFunction{procrefine}{AnytimeRefinement}
    \SetKwProg{myproc}{Procedure}{}{}
    \myproc{\procrefine{$s_{start}$, $s_{goal}$, $\Pi$, $T_{start}$, $T_{bound}$}}{
        CLOSED = INCONS = $\emptyset$ \;
        OPEN = \{$\Pi$\} \tcp*{Insert all states in $\Pi$ with their corresponding g-value}
        $C = g(s_{goal})$; $\epsilon = \max\limits_{s \in \Pi} (\frac{C - g(s)}{h(s)+ \delta})$ \tcp*{$0 < \delta \ll 1$}\;
        \While{$(\GetCurrentTime() - T_{start}) < T_{bound}$ and $ \epsilon > 1$}{
            \While{$(\GetCurrentTime() - T_{start}) < T_{bound}$ and OPEN $\neq \emptyset$}{
                $s_{min} = \min\limits_{s \in OPEN} (g(s) + \epsilon \cdot h(s))$\;
                \If{$s_{min} = s_{goal}$}{
                    $\Pi \leftarrow \ExtractPath()$\;
                    break\;
                }
                insert $s_{min}$ into CLOSED\;
                \For{$s' \in Successors(s_{min})$}{
                    \If{$s'$ was not visited before}{$g(s') = \infty$\;}
                    \If{$g(s') > g(s_{min}) + c(s, s')$}{
                        $g(s') = g(s_{min}) + c(s, s')$ \;
                        \eIf{$s' \notin$ CLOSED}{insert $s'$ into OPEN}{insert $s'$ into INCONS}
                    }
                }
            }
            CLOSED $\leftarrow \emptyset$ \;
            $C \leftarrow  g(s_{goal})$\; 
            $\epsilon \leftarrow \min(\max\limits_{s \in \Pi} (\frac{C - g(s)}{h(s) + \delta}), \max\limits_{s \in \textit{OPEN}} (\frac{C - g(s)}{h(s) + \delta}))$\;
            OPEN $\leftarrow$ OPEN $\bigcup$ INCONS $\bigcup \Pi$\;            
        }
    \KwRet $\Pi$\;}
\end{algorithm}

\subsubsection{Solution Refinement}

In this section, we present the key contribution of the paper.
Given that the initial path is derived from lookups and processes which do not require validity checks, the querying process typically takes up to a few milliseconds (Alg. \ref{alg:query}, lines 2-12). Nonetheless, the allocated time budget is frequently longer and can be utilized to improve the initial solution. The initial paths can be suboptimal since they first move towards the attractor state and then proceed toward the actual goal. Additionally, a more pronounced issue arises in the case of two sequential queries: after completing the first query, the arm must navigate all the way back to its home state before it can proceed to its new goal.

With an initial path and remaining planning time, our objective is to employ this initial solution as an experience for an anytime refinement search (Alg. \ref{alg:refine}).
As presented in Sec. \ref{rel:B}, one approach to conducting an anytime heuristic search involves progressively reducing the heuristic's inflation. We adopt a methodology similar to \cite{anytime_dynamicA*}, with adaptations. 

The anytime-repairing A* algorithm (ARA*) operates through a sequence of weighted A* iterations, gradually reducing heuristic inflation. In standard A*, expanded states are ensured to be consistent and thus not re-expanded. However, in weighted A*, due to heuristic inflation, inconsistent heuristics can be employed in practice, resulting in potential multiple expansions for a single state. ARA* addresses this by limiting expansions to once per state and designating states with improved g-values as \textit{locally-inconsistent}, adding them to an INCONS list. At the conclusion of an iteration and upon achieving a corresponding sub-optimal solution, the subsequent iteration's open-list is initialized with previously identified inconsistent states, enabling the reuse of earlier computational efforts.
A key limitation of ARA* is that it does not provide a guarantee of finding the initial solution within a specific time bound. Moreover, particularly in scenarios involving high-dimensional state spaces, finding the initial solution is often non-trivial.

We observe that given the initial solution, and knowing it is likely to be suboptimal, we consider it to be a collection of potentially inconsistent states. Thus, we initiate the open list of the search with all the states on the path along with their current g-values (Alg. \ref{alg:refine} line 3). This approach facilitates the potential re-expansion of these specific states, thereby enabling the search to discover improved paths towards them. Nevertheless, this process relies on the heuristic inflation value: the higher it is, the less exploration we undertake, given that the goal state already exists in the open list with a heuristic value of zero. We can leverage this observation and balance between exploration and quick refinements. 

As suggested in \cite{seq_align_ara*} and demonstrated in \cite{anytime_dynamicA*}, the suboptimality of a path can be more tightly bounded. Between successive executions searches, a tighter bound can be established using the current path cost $C$. Unlike ARA*, where this insight serves solely to report suboptimality, here we exploit the value of $C$ to determine both the initial weight and its subsequent updates (Alg. \ref{alg:refine}, lines 4 and 24).

Given the current path cost $C$, which is effectively the g-value of the goal state, 
we want at least one state $s$ from the open list or the current path such that:
\begin{equation}
\label{eq:expansion_condition}
    g(s) + \epsilon \cdot h(s) < C
\end{equation}
Which leads us to the following inequality condition
\begin{equation}
\label{eq:expansion_ineq}
    \epsilon < \frac{C - g(s)}{h(s)}
\end{equation}
Inequality \ref{eq:expansion_ineq} signifies that, as the goal state exists within the open list at the beginning of each iteration, in order for it not to be expanded prematurely, the inflation of the heuristic must be less than the above expression. However, as the time available for refinement is determined online, our objective is for each iteration to improve the solution as fast as possible. 

Consequently, for the first iteration, we opt to initialize the weight with the maximum attainable value that still guarantees at least one expansion from the initial path (Alg. \ref{alg:refine} line 4, where $\delta$ assists us in dealing with singularities while ensuring the satisfaction of the inequality condition).

After a search iteration, a new weight must be set for the next iteration, aiming to expand at least one non-goal state. However, re-establishing the weight as in line 4 could risk undermining convergence to an optimal solution. If the path remains unimproved, $C$ and path states persist. Attempting to reset inflation the same way could initiate another search with the same weight, hindering progress. To address this, between iterations, we set the new weight as the minimum between the maximum current path value and the maximum OPEN list value (Alg. \ref{alg:refine} line 24). 

\subsection{Anytime Refinement: Theoretical guarantees}
Given Infinite time budget, we want to show that our algorithm converges to the optimal solution. 
\begin{lem}
    \label{lem:limit}
    Assuming a consistent (monotonic) heuristic function $h:\mathcal{X} \rightarrow \mathbb{R}$ (which implies admissibility, i.e., it never overestimates the cost to reach the goal), and denoting the optimal solution cost as $C^*$, we observe that if, at any given iteration, $\epsilon = 1$, then the resultant cost is $C = C^*$.
\end{lem}

\begin{proof}
    The proof is straightforward. As A* search with a consistent heuristic function is guaranteed to return an optimal solution, as indicated in \cite{anytime_dynamicA*}, when the heuristic inflation reaches 1, our algorithm effectively operates as A*, ensuring the attainment of the optimal solution. 
\end{proof}

\begin{lem}
    \label{lem:monoton}
    The update rule for the inflation of the heuristic 
    $$\epsilon = \min(\max\limits_{s \in \Pi} (\frac{C - g(s)}{h(s) + \delta}), \max\limits_{s \in \textit{OPEN}} (\frac{C - g(s)}{h(s) + \delta})), \hspace{8pt} 0 < \delta \ll 1$$
    is \textbf{strictly decreasing}
\end{lem}

\begin{proof}
    Consider $\epsilon_i$ as the current heuristic weight, and $\epsilon_{i+1}$ as the weight for the next iteration. Furthermore, let $C$ represent the current iteration cost. We know that:
    $$g(s) + \epsilon_i \cdot h(s) \geq C , \hspace{8pt} \forall s \in \textit{OPEN}$$
    Additionally, let $\delta > 0$. Thus,
    $$\epsilon_i > \frac{C - g(s)}{h(s) + \delta}, \hspace{8pt} \forall s \in \textit{OPEN}$$
    This leads us to the result:
    $$\epsilon_i > \max\limits_{s \in \textit{OPEN}}\frac{C-g(s)}{h(s) + \delta} \geq $$ 
    $$\geq \min(\max\limits_{s \in \Pi} (\frac{C - g(s)}{h(s)+ \delta}), \max\limits_{s \in \textit{OPEN}} (\frac{C - g(s)}{h(s)+ \delta})) = \epsilon_{i+1}$$
\end{proof}

\begin{thm}
    Alg. \ref{alg:refine} is guaranteed to converge to the optimal solution.
\end{thm}
\begin{proof}
    The proof follows directly from Lemma \ref{lem:limit} and \ref{lem:monoton}. 
\end{proof}

%% file: experiments.tex
\section{Experiments}
To assess our approach, we conducted an assembly experiment using a 6-DOF robotic arm in simulation and demonstrated its capabilities in both simulated and real-world scenarios\footnote{See supplementary material for accompanying videos.}. 
\subsection{Manipulation Implementation: A Case Study on Assembly}
\label{sec:manipulation}
Let us have a robotic arm $\mathcal{R}$. Let its configuration space denoted as $\mathcal{C}$ and its task space be $\Gamma \subset SE(3)$. In our assembly task, we consider the general task of pick-and-place. Since we are constantly changing the manipulated regions by removing and placing objects, we define the local-RoIs as the \textit{pre-grasp} and \textit{pre-place} regions in $\Gamma$. Nonetheless, a path is being computed in $\mathcal{C}$. 

\begin{figure*}
      \centering
      \begin{subfigure}{0.24\linewidth}
        \centering\includegraphics[width=\linewidth]{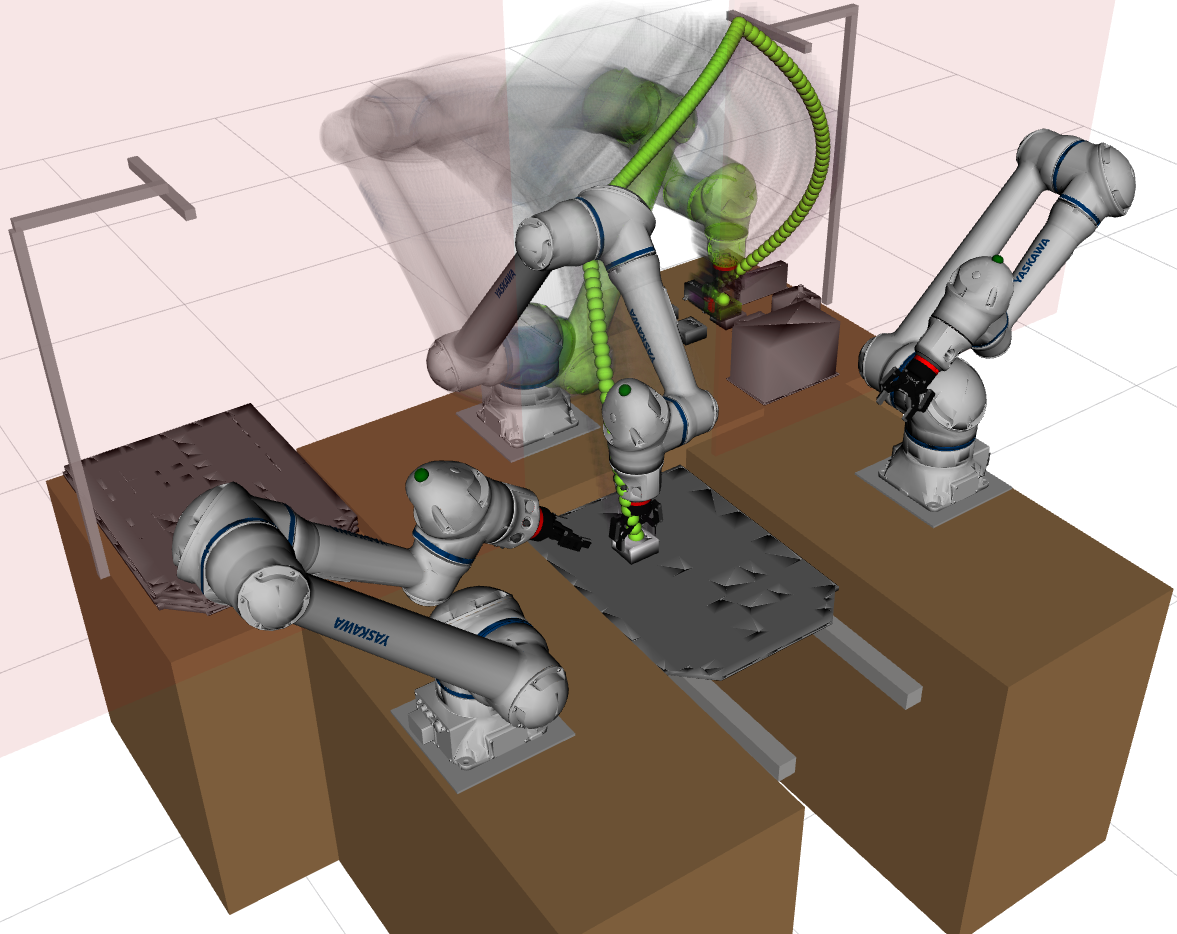}
        \caption{General view: \textbf{without} refinement}
        \label{fig:fig1} 
      \end{subfigure}%
      \hfill
      \begin{subfigure}{0.24\linewidth}
        \centering\includegraphics[width=\linewidth]{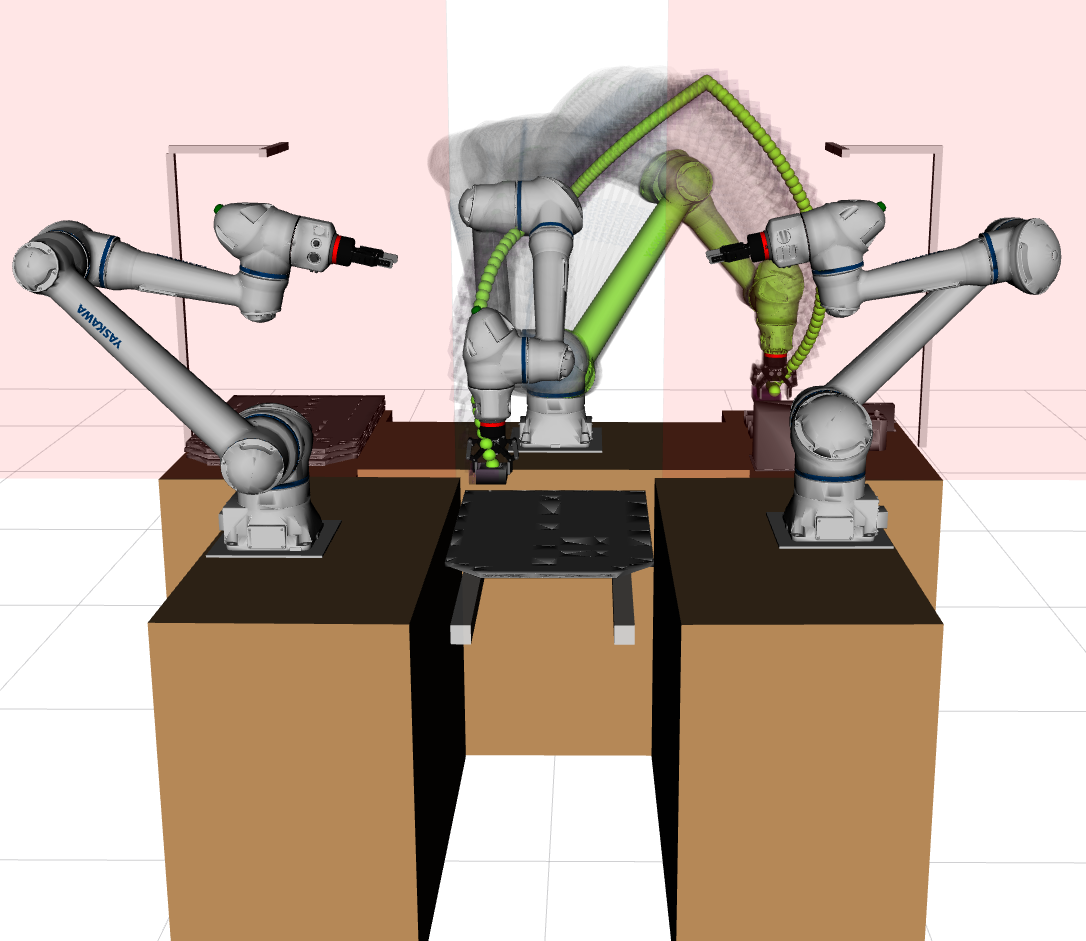}
        \caption{Front view: \textbf{without} refinement}
        \label{fig:fig2}
      \end{subfigure}
      \hfill  
      \begin{subfigure}{0.24\linewidth}
        \centering\includegraphics[width=\linewidth]{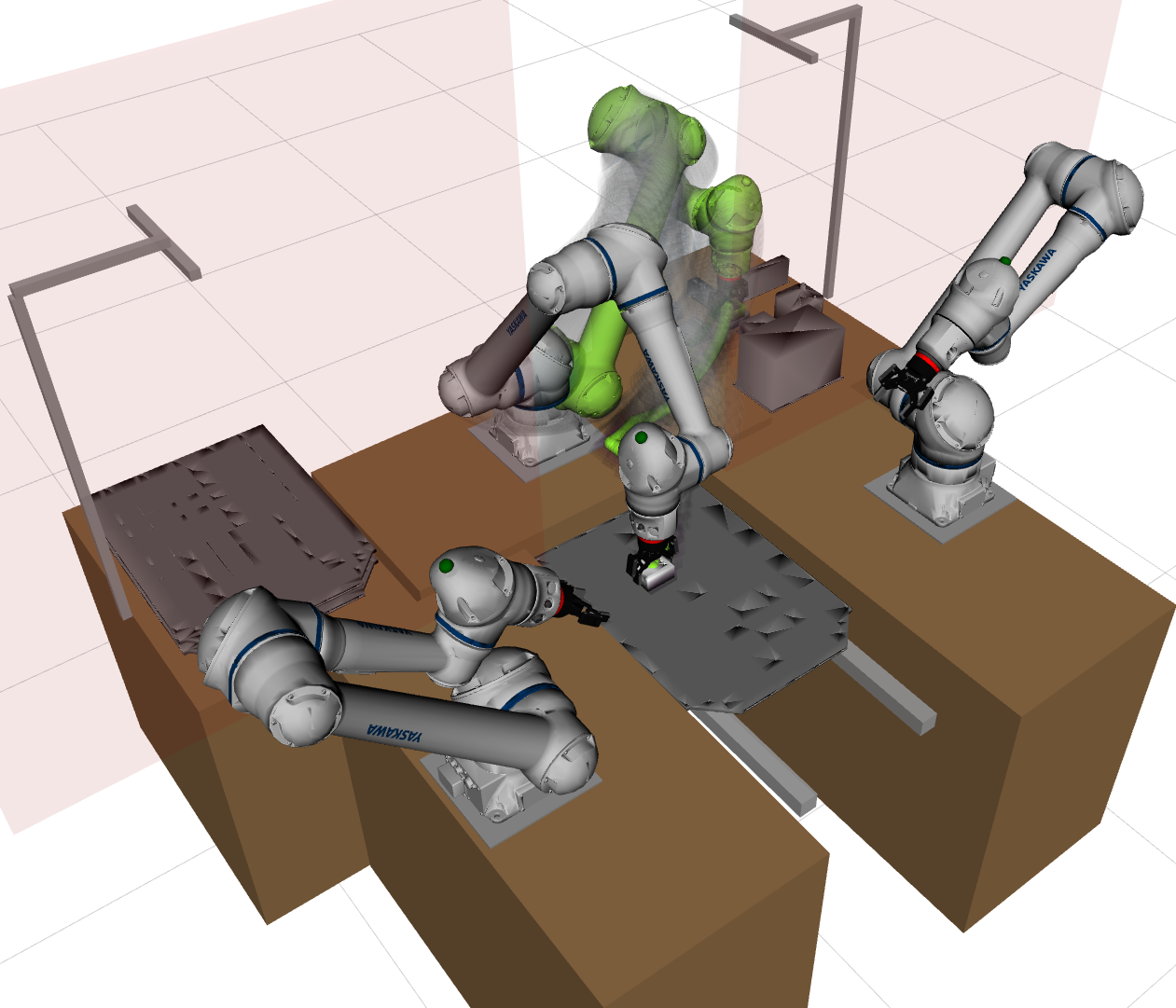}
        \caption{General view: \textbf{with} refinement}
        \label{fig:fig3} 
      \end{subfigure}%
      \hfill
      \begin{subfigure}{0.25\linewidth}
        \centering\includegraphics[width=\linewidth]{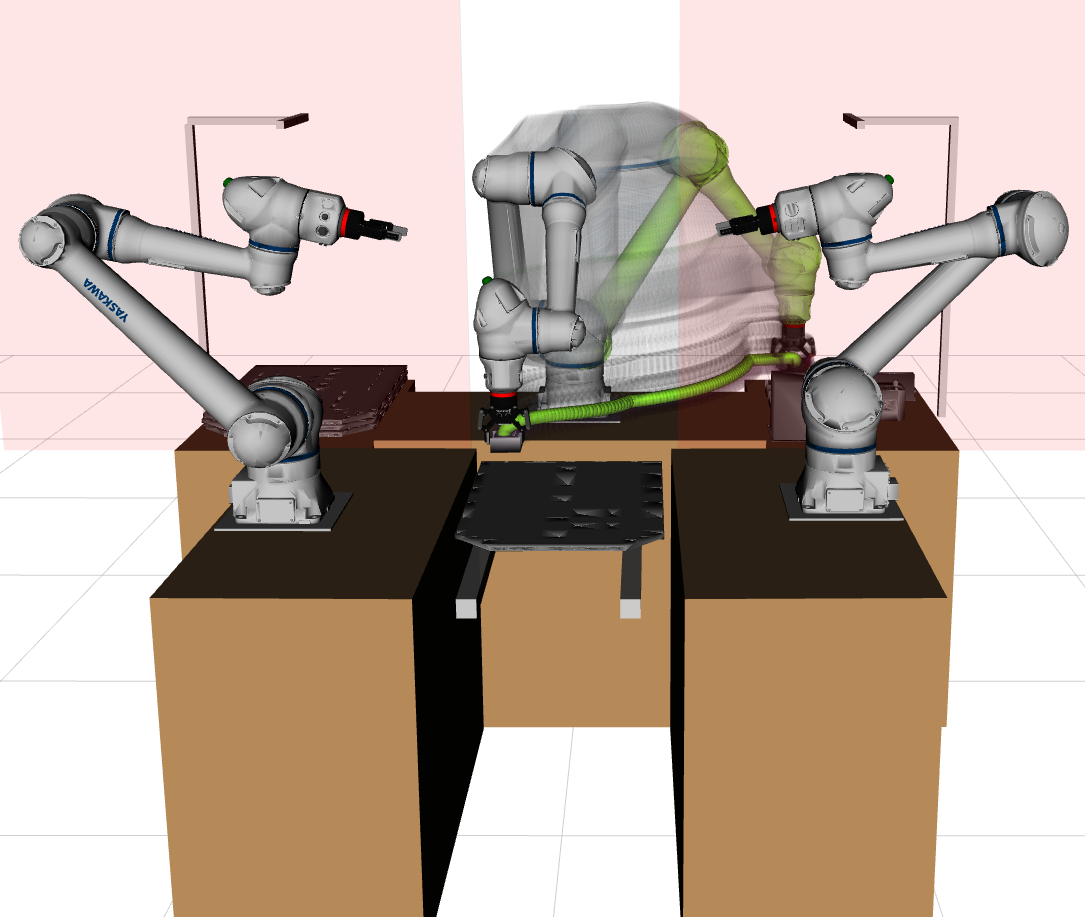}
        \caption{Front view: \textbf{with} refinement}
        \label{fig:fig4} 
      \end{subfigure}
      
    % \centering
    % \includegraphics[width=0.26\linewidth]{figures/top_side_views/no_experience_topview.png}
    % \hfill
    % \includegraphics[width=0.23\linewidth]{figures/top_side_views/no_experience.png} 
    % \hfill
    % \includegraphics[width=0.26\linewidth]{figures/top_side_views/experience_topview.png}
    % \hfill
    % \includegraphics[width=0.23\linewidth]{figures/top_side_views/experience.png}
    % \hfill
    \caption{Comparing CTMP with and without refinement. Fig. \ref{fig:fig1} and \ref{fig:fig2} show general and front views of the path from pick pose (green) to place pose (solid gray), without refinement, having to pass through $s_{home}$. The green line is the end-effector trajectory. \textcolor{blue}{Planning time: 38 msec, cost: 82 steps, suboptimallity: 20.4}. Fig. \ref{fig:fig3} and \ref{fig:fig4} show general and front views with anytime refinement. Here, the path do not pass through $s_{home}$. \textcolor{blue}{Refining time: 1700 msec, cost: 32 steps, suboptimallity: 4.9}.}
    \label{fig:anytime_refinment_action}
\end{figure*}

\begin{table*}[!ht]
    \caption{Experimental results comparing our method with 7 different planners}
    \label{tab:results}
    \centering
    \resizebox{\textwidth}{!}{%
    \begin{tabular}{|c|c|c|c|c|c|c|c|c|c|}
    \hline
    \rowcolor[HTML]{C0C0C0} 
    \cellcolor[HTML]{656565} &  & \textbf{Our Method} & \multicolumn{1}{l|}{\cellcolor[HTML]{C0C0C0}Shortcut-CTMP} & CTMP & PRM & RRT* & ARA* & E-Graphs & RRT-Connect \\ \hline
    \rowcolor[HTML]{EFEFEF} 
    \cellcolor[HTML]{EFEFEF} & Success rate {[}\%{]} & {\color[HTML]{3531FF} \textbf{100}} & 100 & 100 & 97 & 84 & 75 & 90 & 96 \\ \cline{2-10} 
    \rowcolor[HTML]{EFEFEF} 
    \cellcolor[HTML]{EFEFEF} & Avg. Path Cost {[}steps{]} & {\color[HTML]{3531FF} \textbf{16.7}} & 27.8 & 44.1 & 55.3 & 49.4 & 22.4 & 18.3 & 58.7 \\ \cline{2-10} 
    \rowcolor[HTML]{EFEFEF} 
    \multirow{-3}{*}{\cellcolor[HTML]{EFEFEF}\begin{tabular}[c]{@{}c@{}}Using $s_{home}$ \\ as start state\end{tabular}} & Planning time {[}msec{]} & $500 \pm 0$ & $244 \pm 223$ & $17 \pm 1$ & $90 \pm 5$ & $500 \pm 0$ & $500 \pm 0$ & $72 \pm 52$ & $82 \pm 37$ \\ \hline
     & Success rate {[}\%{]} & \cellcolor[HTML]{FFFFFF}{\color[HTML]{3531FF} \textbf{100}} & \cellcolor[HTML]{FFFFFF}100 & \cellcolor[HTML]{FFFFFF}100 & \cellcolor[HTML]{FFFFFF}98 & \cellcolor[HTML]{FFFFFF}82 & \cellcolor[HTML]{FFFFFF}50 & \cellcolor[HTML]{FFFFFF}72 & \cellcolor[HTML]{FFFFFF}96 \\ \cline{2-10} 
     & Avg. Path Cost {[}steps{]} & \cellcolor[HTML]{FFFFFF}{\color[HTML]{3531FF} \textbf{35.4}} & \cellcolor[HTML]{FFFFFF}48.1 & \cellcolor[HTML]{FFFFFF}88.2 & \cellcolor[HTML]{FFFFFF}66.1 & \cellcolor[HTML]{FFFFFF}76.0 & \cellcolor[HTML]{FFFFFF}28.8 & \cellcolor[HTML]{FFFFFF}75.5 & \cellcolor[HTML]{FFFFFF}80.3 \\ \cline{2-10} 
    \multirow{-3}{*}{\begin{tabular}[c]{@{}c@{}}Using random \\ potential state \\ $s \in \Phi$ as start state\end{tabular}} & Planning time {[}msec{]} & \cellcolor[HTML]{FFFFFF}$2089 \pm 624$ & \cellcolor[HTML]{FFFFFF}$459 \pm 16$ & \cellcolor[HTML]{FFFFFF}$37 \pm 3$ & \cellcolor[HTML]{FFFFFF}$92 \pm 1$ & \cellcolor[HTML]{FFFFFF}$1857\pm 720$ & \cellcolor[HTML]{FFFFFF}$1823\pm 713$ & \cellcolor[HTML]{FFFFFF}$247 \pm 12$ & \cellcolor[HTML]{FFFFFF}$116 \pm 1$ \\ \hline
    \end{tabular}}
\end{table*}

\subsubsection{State validity check} 
Considering that neighborhoods are constructed within $\Gamma$, the validation process for task space states within $\mathcal{G}$ involves verifying the existence of an inverse kinematics solution and then assessing graspability and placeability (i.e. evaluating the feasibility of grasping and placing operations from the given state). This is accomplished by generating grasp and pose action trajectories through Runge-Kutta integration.

\subsubsection{Anytime refinement planning from any potential state to a goal state}
As assembly is a sequential process, the planner continuously receives consecutive queries, where any potential state $s\in \Phi$ can serve as the initial state for the planning request. From any potential state, there exists a path to $s_{home}$. In assembly, the start state can be a previous goal state, a state along a previously queried path, or $s_{home}$. This enables us to simplify Alg. \ref{alg:query} even more, since lines 7-10 are now trivially computed; $\Pi_{home, s}$ is equal to or a subset of the previous queried path. Still, the concatenated path may result in a highly suboptimal solutions without the refinement approach, as depicted in Figure \ref{fig:anytime_refinment_action}.

\subsubsection{Anyobject Extension}
% In manipulation, particularly during assembly tasks, there is a constant interplay between grasping and placing objects. This dynamic alters the arm's kinematic chain and therefore its collision model. To ensure both completeness and collision-free solutions, our algorithm is extended to accommodate any-object manipulation. Assuming provided with object primitives and their maximum allowable sizes (sphere, box, cylinder, etc.). In Algorithm \ref{alg:preprocess}, line 7, a set of paths is computed. Initially, we compute a path without any object attached. Subsequently, for each object primitive, we determine the maximum size that can be attached while remaining collision-free. We then recompute the path with the object attached (at a size that previously caused collisions) and re-evaluate the maximum allowable size. This iterative process continues for each object, culminating in a set of representative paths. During online queries, we simply search for the object primitive and size best encompassing the queried object.

In manipulation, particularly during assembly tasks, there is a constant interplay between grasping and placing objects which alters the arm's kinematic chain and therefore its collision model. To ensure collision-free solutions, our algorithm is extended to accommodate any-object manipulation. Having a set of object primitives and their maximum allowable sizes (sphere, box, cylinder, etc.), in Algorithm \ref{alg:preprocess}, line 7, a set of paths is computed. Initially, we calculate a path without any object attached. Subsequently, for each object primitive, we determine the largest collision-free dimensions that can be accommodated by that path. The trajectory is recomputed with the colliding dimensions and this process is repeated until the object reaches the maximum allowable size. We repeat these steps for each object, culminating in a set of object-path pairs which we can query online given the object primitive and size.

\subsection{Evaluation}
In this experimental setting, we preprocessed 4 local-RoIs, allowing primitives in the $x$, $y$, and $z$ dimensions, along with a single rotational degree of freedom (roll, pitch, or yaw). The preprocessing stage (Alg. \ref{alg:preprocess}) took 126 minutes and used 88 MB of memory for the entire RoI.
We compared our method with seven algorithms: CTMP without anytime refinement, CTMP with shortcutting, PRM, RRT*, ARA*, E-Graphs and RRT-Connect, presented in Table \ref{tab:results}. All experimental results were run on Intel Core i9-12900H with 64GB RAM (4.7GHz). The top three rows provide statistics for 100 queries using the home state ($s_{home}$) as the initial state. A timeout of 500 milliseconds was set for all planners. The bottom 3 rows provide statistics for 50 queries between random potential states ($s_{start}\in\Phi$) from a pick region as initial states to random goal states in place region. Here, we randomly assigned extra planning time from 500 msec to 3 sec, mirroring real scenarios where planners query during ongoing execution. The ARA* search was initiated with a weight of 50, which was then gradually decreased by 5. For the E-graphs approach, we utilized 8 experiences from the CTMP preprocess phase, selecting 2 experiences from each local-RoI. Each of these experiences represented a path leading to a central state within a randomly chosen neighborhood. In both experiments, while maintaining a 100\% success rate, our planner also notably enhanced the solution quality. The results demonstrate that our planner not only ensures provably constant-time completeness but also generates quality solutions when the allocated time budget permits.

%% file: conclusions.tex
\section{Conclusions}
We have introduced an anytime adaptation of CTMP, demonstrating its capability to converge towards optimal solutions. Its applicability in manipulation tasks was showcased in the context of assembly's pick-and-place operations. Future research could leverage the framework of CTMP and utilize its solutions as experiences for making real-time adjustments to minor environment changes. 

% Furthermore, a development involves refining the solution through contract search methods (deadline-aware search), which strives to achieve the best solution within a time-frame, may prove valuable given time-critical planning tasks.